%% file: main.tex
\documentclass[conference]{IEEEtran}
\IEEEoverridecommandlockouts
\usepackage{cite}
\usepackage{amsmath,amssymb,amsfonts, amsthm}
\usepackage{multirow}

\DeclareMathOperator*{\argmax}{argmax}
\DeclareMathOperator*{\argmin}{argmin}

\input{macros}

\newtheorem{theorem}{Theorem}

\newtheorem{definition}{Definition}
\newtheorem{example}{Example}
\newtheorem{proposition}{Proposition}
\newtheorem{assumption}{Assumption}

\usepackage{algorithm}
\usepackage{algpseudocode}
\usepackage{subcaption}

\usepackage{graphicx}
\usepackage{textcomp}
\usepackage{xcolor}
\def\BibTeX{{\rm B\kern-.05em{\sc i\kern-.025em b}\kern-.08em
    T\kern-.1667em\lower.7ex\hbox{E}\kern-.125emX}}
\begin{document}

\algrenewcommand\algorithmicrequire{\textbf{Input:}}
\algrenewcommand\algorithmicensure{\textbf{Output:}}

\title{Correcting Learning-based Perception for Safety\\
}
\newcommand{\sayan}[1]{\textcolor{blue}{#1}}
\newcommand{\yangge}[1]{\textcolor{purple}{#1}}

\author{\IEEEauthorblockN{Yan Miao}
\IEEEauthorblockA{\textit{UIUC} \\
}
\and
\IEEEauthorblockN{Hussein Darir}
\IEEEauthorblockA{\textit{UIUC} \\
}
\and
\IEEEauthorblockN{Sayan Mitra}
\IEEEauthorblockA{\textit{UIUC} \\
}
}

\maketitle

\begin{abstract}
Learning-enabled  perception is important in many autonomous systems. Unlike traditional sensors, the boundary where ML perception does or does not work is poorly characterized. Incorrect perception can lead to unsafe or overtly conservative downstream control actions. In this paper, we propose a two-step strategy for correcting ML-based state estimation. First, an offline computation is used to characterize the uncertainties resulting from the  ML module's state estimation,  using preimages of perception contracts. Second, at runtime, a risk heuristic is used to choose particular states from the uncertain estimates to drive the control decisions. We perform extensive simulation-based  evaluation of this runtime perception correction strategy on different vision-based adaptive cruise controllers (ACC modules), in different weather conditions, and road scenarios.  Out of 45 ACC scenarios where the original perception-based control system using Yolo and LaneNet led to safety violations,  in 73\% of the scenarios, our runtime perception correction preserved safety; our method wouldn't be able to recover 27\% of the scenarios where the construction of the preimages of perception contracts is not fully conformant. Further, our runtime perception correction strategy is not overly conservative---on the average only a 2.8\% increase in completion time is experienced in the corrected scenarios, with mild interventions. 
\end{abstract}


\input{section/1.introduction}
\input{section/2.related_work}
\input{section/4.system_setup}
\input{section/5.methodology}
\input{section/6.experiments}
\input{section/7.discussions}

\bibliographystyle{IEEEtran}
\bibliography{IEEEabrv,references}
\end{document}

%% file: macros.tex
\newcommand{\U}{\mathcal{U}}

\newcommand{\corrSys}{S_{\mathit{MJ}}}
\newcommand{\D}{\mathcal{D}}
\newcommand{\E}{\mathcal{E}}

\newcommand{\N}{\mathcal{N}}

\newcommand{\pc}{M}
\newcommand{\R}{\mathbb{R}}

\newcommand{\Unsafe}{\mathit{Unsafe}}
\newcommand{\Safe}{\mathit{Safe}}

\newcommand{\X}{\mathcal{X}}
\newcommand{\Y}{\mathcal{Y}}

%% file: section/1.introduction.tex
\section{Introduction}
\label{sec:intro}

Machine learning (ML) can play an important role in the creation of  cyber-physical and autonomous systems that operate in complex environments. Inexpensive sensors coupled with powerful pre-trained ML models can serve as an attractive alternative to traditional sensing and state estimation methods. At the same time, it is also well-known that ML models suffer from fragile decision boundaries and adversarial examples~\cite{wiyatno2019adversarial}. Indeed, a major AI safety concern is the potentially out-sized impact of this lack of robustness in safety critical applications. On the other hand, for traditional  control systems and cyber-physical systems (CPS), there is a rich body of techniques for  model-based design and analysis of systems that are robust to certain types of disturbances~\cite{ML:CDC04-full,verifai-cav19,Anta:2010,ChanM17ARCH}.
These methods provide rigorous guarantees about safety, robustness, and stability, but only in relatively  structured environments and with simple sensor models. In this paper, we explore the middle-ground and aim to provide  semi-formal safety guarantees for AI-enabled CPS.

\begin{figure}
    \centering
    \includegraphics[width=\linewidth]{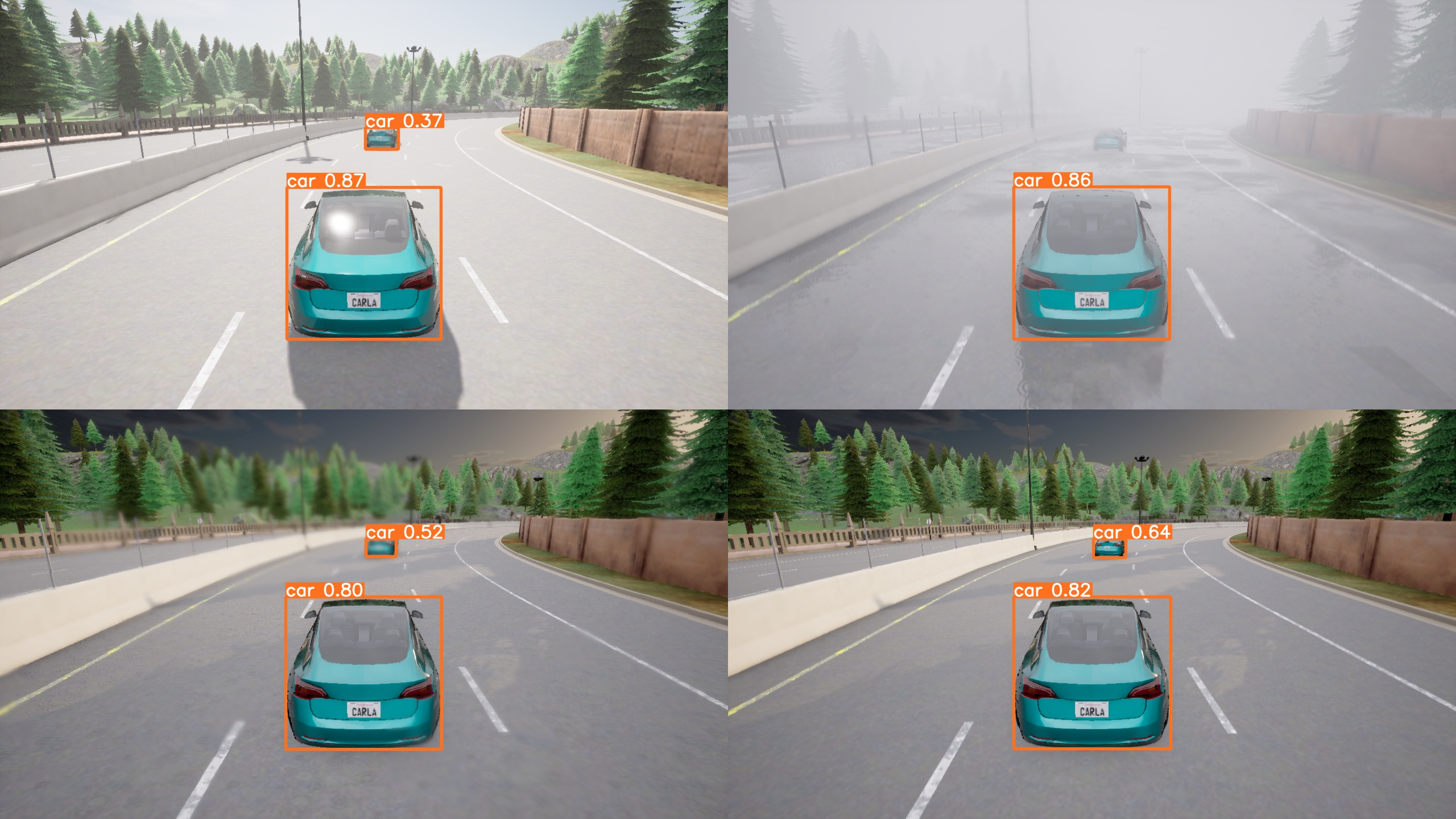}
    \caption{\small Screenshot of simulation of Autonomous Cruise  Control (ACC) scenarios under different weather conditions. Under rain and fog, ML-enabled perception designed to estimate crosstrack error and distance to leading vehicles, can have larger errors. For example, it cannot detect a  vehicle in the the top-right scenario, and the confidence score (labeled outside the orange box) for detecting nearby vehicle is lower. Our approach corrects for such perception errors for a family of controllers.}
    \label{fig:combined_img}
\end{figure}


Consider a vision-based Autonomous Cruising Control (ACC) system in which a  vehicle (controller) relies on perception for lane keeping and maintaining safe distance from a leading vehicle. The perception module $h$ (specifically Yolo~\cite{reis2023realtime} and LaneNet~\cite{10.1109/IVS.2018.8500547} in this example) provides estimates of distance to the leading vehicle (provided there is one) and crosstrack error with respect to the lane center. The vehicle controller $g$ uses these observations or estimates to compute the steering, throttle, and brake inputs for the vehicle (see Figure~\ref{fig:original-system}).
Setting aside the ML-based state estimator $h$ for a moment, we observe that the rest of this system is a classical cyber-physical system (CPS). If only we could assume that the state estimator were perfect $h^*$ or that it came with  reasonable error bounds,  then a whole arsenal of tools would become available for design and analysis: We could get stability envelops using Lyapunov analysis, we could compute invariants  using reachability, and so on. However, like other machine learning models, $h$ as implemented in Yolo and LaneNet do not have error specifications and is fragile. Further, its output estimates depend on environmental factors like lighting and weather in complex ways, which can violate safety of ACC (see Figure~\ref{fig:combined_img}). This near slip from grasp motivates us to investigate the following problem: Given a controller $g$ that preserves some invariant $R$ with perfect perception $h^*$, and given a real perception module $h$ (implemented using ML and thus afflicted by fragility, environment, lack of specs), can we modify or correct the output of $h$  so that the resulting system preserves the safety invariant $R$. We call this the  {\em perception correction problem} and the formal statement appears in Definition~\ref{def:problem}.

We propose a solution to this problem using  {\em preimages of perception contracts\/}.
A perception contract $M$~\cite{EMSOFT,OOPSLA} bounds the 
output of an ML-based state estimator {\em as a function of the ground truth state\/}, so that  it preserves a closed-loop invariant such as $R$ (see Definition~\ref{def:control-invariant-set}). Different representations for  perception contracts have been proposed using  piece-wise affine set-valued functions~\cite{EMSOFT} and decision trees~\cite{OOPSLA}. 
Perception contracts have been used to verify a vision-based lane-keeping system~\cite{EMSOFT}, automated landing for a drone~\cite{sun2023learningbased}, and distributed formation flight~\cite{hsieh2023assuring}. 
A closely related notion of weakest preconditions has been used to analyze vision-based taxiing in the probabilistic setting~\cite{10.1007/978-3-031-37706-8_15}. 
While these ideas have been fruitful for offline verification, at runtime, the ground truth state is not available, and therefore, perception contracts cannot be  used for monitoring or for taking corrective actions.

Our proposed method uses the simple observation that the {\em preimage $M^{-1}$ of a preception contract (PPC)}---which outputs the uncertainty in  state given an input observation---can be used at runtime. Specifically, if the set of states $M^{-1}(y)$ corresponding to an observation $y$ shows no risk of violating the target invariant $R$, then no corrective action is needed. On the other hand if $M^{-1}(x)$ could potentially compromise safety, then some corrective action may be necessary. 

Calculating the preimage perception contract (PPC) directly is a complex task, and hence, we use a Bayesian Neural Network (BNN) to construct the PPC based on data. 

We then introduce a concept of risk, which defines a monotonous function relative to the safety property across different states. Our proposed risk heuristic, guides control actions by prioritizing the states within the inferred uncertain set that pose the highest level of risk, as illustrated in Figure~\ref{fig:perception_correct}

By integrating the correction module at runtime, we were able to effectively recover 33 out of the 45 unsafe scenarios, resulting in a 73\% recovery rate. It is important to note that our inability to address all unsafe scenarios may be attributed to conformance violations during the empirical construction of the preimage perception contract (PPC) from data.

Additionally, our evaluations demonstrate that the intervention of the module is minimal and does not impose excessive or overly conservative control measures. In fact, our experiments reveal only a 2.8\% increase in the time required to complete tasks when our module is integrated.

In summary, our approach to addressing the runtime perception correction problem involves the incorporation of the preimage perception contract (PPC) and a risk heuristic into the existing closed-loop system. Initial findings have yielded promising results, indicating that this method is effective in enhancing the system's  safety.

\vspace{1em}

%% file: section/2.related_work.tex
\section{Related Work}

With the emergence of increasingly sophisticated sensors, such as cameras, LiDAR, and radar, coupled with the development of advanced perception algorithms, the issue of seamlessly integrating these sensors and their associated machine learning-based algorithms into the controller pipeline has become a prominent subject of research. In recent studies, innovative methods like imitation learning~\cite{DBLP:journals/corr/abs-2103-17118} and reinforcement learning utilizing RGB cameras~\cite{DBLP:journals/corr/abs-1801-05299} have been introduced to tackle the challenge of vision-based control for autonomous vehicles. However, it's important to note that these approaches are data-driven, and they do not effectively characterize or bound the potential errors in perception, which ultimately limits their capacity to guarantee safety.

Recent research efforts have been primarily directed towards ensuring safety through vision-based control. In the work by Dean et al. \cite{dean2020robust}, the authors synthesized a vision-based controller for autonomous vehicles and carried out theoretical analyses to establish a robust safety guarantee. However, their approach involved simplifying the vehicle model to a linear one.
In a subsequent study, as presented in \cite{dean2021guaranteeing}, the author proposed a Measurement-Robust Control Barrier Function (MR-CBF) that incorporates an optimization method for synthesizing a safe controller. 
Dawson et al.~\cite{dawson2022learning} focused on working with high-dimensional sensors like LiDAR. They proposed a method for learning a control Lyapunov function (CLF) and a control barrier function (CBF) within the observation space, without making assumptions about the perception module. 
Additionally, in the work by Chou et al. \cite{sun2023learningbased}, authors use the concept of perception contract to design a controller for a safe landing problem. \cite{chou2022safe}, the authors designed a neural network-based perception module capable of outputting a set of potential states. Subsequently, they applied contraction theory and robust motion planning algorithms to synthesize a robust and safe vision-based controller. This work is closely related to our research; however, our approach involves generating a set of potential states by learning the behavior of a black-box perception model, in contrast to their method, which has to constructs such set of potential states while designing the perception model from data.

\vspace{1em}

%% file: section/4.system_setup.tex
\section{Runtime Perception Correction}
\label{sec:system}

In this section, we introduce the different parts making up  the  perception-based control system and then define the runtime perception correction problem. 

\subsection{Perception-Based Control System}
\label{sec:sysmodel}

The closed-loop system, comprises of three components: the plant with dynamics $f$, the controller $g$, the learning-based perception module $h$. 

\subsubsection*{Plant dynamics}
The state of the physical part of the system is denoted by vector $x \in \X \subset \mathbb{R}^n $, where $\X$ is called the state space. We denote by $x[i]$  the $i^{th}$ component of $x$. For example, for an autonomous vehicle, the state vector $x$ may include its position, velocity, heading, distance to front vehicle, etc.

System-level safety requirements are given in terms of a set of {\em unsafe states\/}, $\Unsafe \subset \X$, that the overall system must stay away from. The set of safe states,  $\Safe = \X \setminus \Unsafe$, is the complement of the unsafe states.

\begin{example}
    Consider a vehicle (ego) following curvy lanes on a highway with objective of ensuring that the vehicle stays within its lane and does not deviate. The state vector is defined by valuations of several variables: $\{p_E, v_E, \theta, d_L\}$, where $p_E$ is ego vehicle's pose (position and heading), $v_E$ is ego vehicle's velocity, $\theta$ is the angle between ego's heading and the lane's heading, $d_L$ is vehicle's cross-track error with respect to the center of the lane.
    Since the vehicle enters an unsafe state when a lane departure occurs, then we can define the {\em unsafe states\/} as a set, $\Unsafe = \{(p_E, v_E, \theta, d_L) : |d_L| \geq \frac{L}{2} \}$, where $L$ is the lane width. 
    \label{example:autonomous-vehicle}
\end{example}

The evolution of the plant state is described by a {\em  dynamic function} 
$f: \X \times \U \mapsto \X $, where $\U$ is the {\em control input space\/}. 
In Example~\ref{example:autonomous-vehicle}, the control inputs for the  vehicle are  throttle $t \in [0, 1]$, brake $b \in [0, 1]$, and steering $s \in [-1, +1]$, here $-1$ stands for the  maximum left steering input and $+1$ stands for maximum right steering. Given a state $x \in \X$ and an input $u \in \U$, the  next state of the vehicle $x_{t+1} = f(x_t, u_t)$. 

\subsubsection*{Perception and control}
The method developed in this paper targets systems in which the control input $u$ is computed in two stages: first, a {\em perception module\/} $h$ interprets the  signals generated in state $x$ via sensors to produce an observation $y \in \Y$, and then, the  {\em controller \/} $g: \Y \rightarrow \U$ takes as input this observation  $y$  and computes the control input (for the plant). 
A diagram is shown in Figure~\ref{fig:original-system}. In Example~\ref{example:autonomous-vehicle}, $h$ could output cross-track-error $d_L$ as the observation $y$, and the controller could applies a hard brake whenever the $d_L$ is above some threshold.

The perception module $h$ generates observation $y$ from the actual plant state $x$.
We name a perfect observer $h^*: \X \mapsto \Y$. In Example~\ref{example:autonomous-vehicle}, if $x$ is known, the observation $y$ (cross-track error $d_L$) can be directly obtained from $x$ by dropping the extra state components and retaining $d_L$ as observation. However, in most autonomous systems, usually state information $x$ cannot be directly obtained. Therefore we rely on an observer $h$, which encapsulates the behavior of the sensors that generate the raw signals (e.g., images, LIDAR returns) as well as the algorithms (e.g., machine learning models, filters), to generate the observation $y$ from those signals. The signals 
also critically  depend on certain environmental factors (e.g., lighting, fog, rain, etc.). The space of all possible such environmental conditions is denoted by $\E$. Thus, the perception module is modeled as a function $h:\X \times \E \rightarrow \Y$.
In Example~\ref{example:autonomous-vehicle}, to achieve the lane following task, ego vehicle needs to first rely on the camera sensor to generate an RGB image. Then the analysis of this image is done using a ML algorithm (E.g. LaneNet \cite{10.1109/IVS.2018.8500547}), to produce lane information and subsequently cross-track-error $d_L$.

There are several reasons for  this two-stage architecture for the computation of $u$. First, from the control theory point of view, it is standard to think of the whole pipeline as the composition of a state estimator ($h$) and a controller ($g$). Loosely speaking, the certainty equivalence principle assures that the optimality of controller design can be preserved by this decomposition, under appropriate assumptions.
The access to privileged state information, like the observables, have also been noted to benefit the development of reinforcement learning-based controllers~\cite{chen2021general}. 

Identifying all possible environmental factors that influence $h$ can be a complex problem. This work is based on the premise that domain experts prescribe the dominant factors in $\E$ with respect to which runtime perception correction should be applied.

\subsubsection*{Closed-loop system}
The discrete time evolution of the  {\em closed-loop system\/} or simply the {\em system} $S$, in an environment $e \in \E$, as shown in Figure~\ref{fig:original-system}, is given by the following:
\begin{equation}
    x_{t+1} = f(x_t, g(h(x_t, e)). 
    \label{eq:sys}
\end{equation}

An {\em execution} in an environment $e$, is a sequence of states $\alpha(e) = x_0, x_1, \ldots,$ such that for each $t$, $x_{t+1}$ and $x_t$ satisfy~(\ref{eq:sys}). With respect to an unsafe set $\Unsafe$, the system $S$ is safe 
over an environment $E' \subseteq \E$ and a set of initial states $X_0 \subseteq \X$, if for each $e \in E'$ and $x_0 \in X_0$, none of the states in $\alpha(e)$ are in $\Unsafe$, i.e., the reachable states of $S$ are disjoint from $\Unsafe$.

\begin{definition}
    A {\em control invariant set} $R\subseteq \X$ for the system $S$ is  a set such that:
    \begin{enumerate}
        \item $R \subseteq \text{Safe}$
        \item $\forall x \in R, \; \exists \, u \in \U \; \text{s.t.} \; f(x, u) \in R$.
    \end{enumerate}
    \label{def:control-invariant-set}
\end{definition}

There has been substantial progress in computing the control invariant sets for systems. Notable techniques include  barrier certificates\cite{verification-barrier-certificates, PRAJNA2005526}, formal controller synthesis\cite{badings2022probabilities}, control barrier functions\cite{control-barrier-function}, and more recently neural barrier functions\cite{dawson2022safe}.

These techniques vary in terms of the  levels of knowledge needed about $f, g, h$, their computational complexity, and the level of formal guarantee that they provide. 
However, our system $S$ includes learning-enabled perception $h$, which depends on the environment in complex ways, and therefore, some of the existing techniques will not be directly applicable. Instead, our sufficient condition for proving safety of the overall system is based on using control invariant sets for an {\em idealized  controller-observer pair $g^*,h^*$}. In the two-staged observer-controller design paradigm, it is indeed common for the controller design to assume that the observer is at least asymptotically correct. The following  codifies this assumption about such an idealized observer-controller pair.

\begin{assumption}[Safety with perfect observer-controller] There exists a {\em perfect observer} $h^*:\X \rightarrow \Y$ and controller $g^*: \Y \rightarrow \U$ pair for a given safe invariant set $R \subseteq \Safe$. That is,  for any $x \in R$,
$f(x,g^*(h^*(x)) \in R.$
\label{assump:g_h*}
\end{assumption}

\subsection{Runtime Perception Correction  Problem}
\label{sec:pcp}
Due to environmental uncertainty, sensor noise, and inaccuracies in the Deep Neural Network, the observation $h(x, e)$ may not be accurate and deviate from $h^*(x)$, which could lead to unsafe conditions (e.g., incorrect lane detection causing lane invasion in foggy conditions). As a result, in certain cases, even though the system has a safe controller under perfect perception (Assumption~\ref{assump:g_h*}), there's no guarantee the system is safe with the neural-network based perception module $h$. 

\begin{definition}[Runtime perception correction problem]
\label{def:problem}
\begin{itemize}
    \item[]
\end{itemize}
Given:
\begin{itemize}
    \item a perfect observer-controller pair $g^*, h^*$ and a corresponding  control invariant set $R$ that proves safety of the  closed-loop system $S$  with respect to $\Safe$.
    \item an actual (learning-enabled) observer  module $h$ that depends on environment factors in $\E$.
\end{itemize}
 The objective is to  synthesize a controller $\hat{g}: Y \mapsto U$ such that $\forall x_0 \in \X_0, e \in \E_0$, the new closed loop system with $h$, $g$ is safe. 
\end{definition}

In addition, to ruling out trivial solutions (e.g., always brake and stop), we have a soft-requirement that the new system with $g$ and $h$ should be minimally invasive over $S$. Metrics for invasiveness are not straightforward to define, and we will consider some examples in Section~\ref{sec:case-study}.

\vspace{1em}

%% file: section/5.methodology.tex
\section{Methodology}
Our idea for solving the above problem involves two stages: the first stage infers the uncertainty in the state of the system at runtime from the observations, and the second stage takes control action based on the riskiest states in inferred uncertain set.
For reasons that will become clear below, the first stage is called {\em Preimage of a Perception Contract (PPC)\/} and the latter is called a {\em risk heuristic\/}. 
The difficulty of inferring the uncertainty in the actual state from  observations is addressed using  using the recently invented notion of  perception contracts

\subsection{Perception Contracts $M$}
\label{sec:pc}
Safety analysis of machine learning-enabled and perception-based control systems is hard since we do not have specifications for the ML modules. Specification are not only necessary for formal verification, but they form the basis for modern, large-scale software engineering by enabling unit tests, modular design, and assume-guarantee reasoning. 
In~\cite{EMSOFT,OOPSLA} the authors introduce the notion of {\em perception contracts} for addressing this problem. 
A perception contract $M$ for an actual perception module $h$, in the context of a closed loop system $S$, and its control invariant set $R$, captures two ideas: First, $M$ is an over-approximation of $h$, over at least some part of the relevant environment space $\E$. This is called the {\em conformance\/} of the contract. Second, $M$ preserves system-level correctness of $S$ with respect to $R$. That is, if $M$ is plugged-in to $S$, then the resulting closed-loop system, $S_M$, should preserve $R$. 

 \begin{definition}
\label{def:pc}
For a given perception module $h:\X \times \E \rightarrow \Y$, an invariant set $R\subseteq \X$, and an environment $\E' \subseteq \E$, a  {\em perception contract\/} is a map  $\pc: \X \rightarrow 2^\Y$ that satisfies the two conditions:
\begin{enumerate}
    \item Conformance: $\forall x \in R, e\in E'$, $h(x,e)\subseteq \pc(x)$.
    \item Correctness: $\forall x \in R, f(x,g(M(x)) \subseteq R$.

\end{enumerate}
\end{definition}
Note that here we use $2^\Y$ to denote a powerset of $\Y$.

In the previous works, the authors have shown that it is possible to construct such contracts from data for vision-based lane keeping systems and for automated landing systems. The constructed contracts can indeed be used to rigorously prove system-level safety (e.g., car does not leave the lane boundaries). In~\cite{EMSOFT}, for example,  $M(\cdot)$ is a piece-wise affine set-valued function constructed from data, and the correctness condition is verified using program analysis. 
Owing to the complexity of the actual perception pipeline $h$ and its complex dependence on the environment $E$, the conformance property is empirically validated based on input-output data.

\subsection{Preimage of Perception Contracts $M^{-1}$}
\label{sec:ppc}

Inspired by perception contracts, in this work we explore how such contracts can be used at runtime to possibly correct perception errors. The challenge we face is that the ground truth state $x$ is not available at runtime to use $M(x)$, even though, $M(\cdot)$ is computed offline. Our key idea is to use the {\em preimage} of perception contracts, that is, $M^{-1}:Y \rightarrow 2^X$. Conceptually, 
for a given observation $y \in \Y$, its preimage $M^{-1}(y)$ gives the set of all possible states that could generate $y$, in some environment $\E'$.

For a given perception contract $M:\X\rightarrow 2^\Y$ we define $M^{-1}:\Y\rightarrow 2^\X$  as $M^{-1}(y) := \{ x \ | \ y \in M(x) \}$. It follows that, if that for any realizable $y \in \Y,$ if observer function $h$ conforms to $M$ over $\E'$, then $h^{-1}(y) \subseteq M^{-1}(y)$.
\begin{proposition} 
    \label{prop:ppc-ph}
    Consider any realizable observation $y \in \Y$, such that there exists $x \in \X, e \in \E'$ with $h(x,e) = y$. If $h$ conforms to the perception contract $M$, then $h^{-1}(y) \subseteq M^{-1}(y)$.
\end{proposition}
\begin{proof}
    Follows from the definitions. Consider any realizable $y \in \Y$ and let $h^{-1}(y) := \{ x \ | \ \exists e \in \E', h(x,e) = y \}.$ Consider any $x \in h^{-1}(y)$. Since, $h$ conforms to $M$ over $\E'$, $h(x,e) = y \in M(x)$, for some $e \in \E'$. By definition of $M^{-1}$, then $x \in M^{-1}(y)$.
\end{proof}

\begin{figure}
    \centering
    \includegraphics[width=0.7\linewidth]{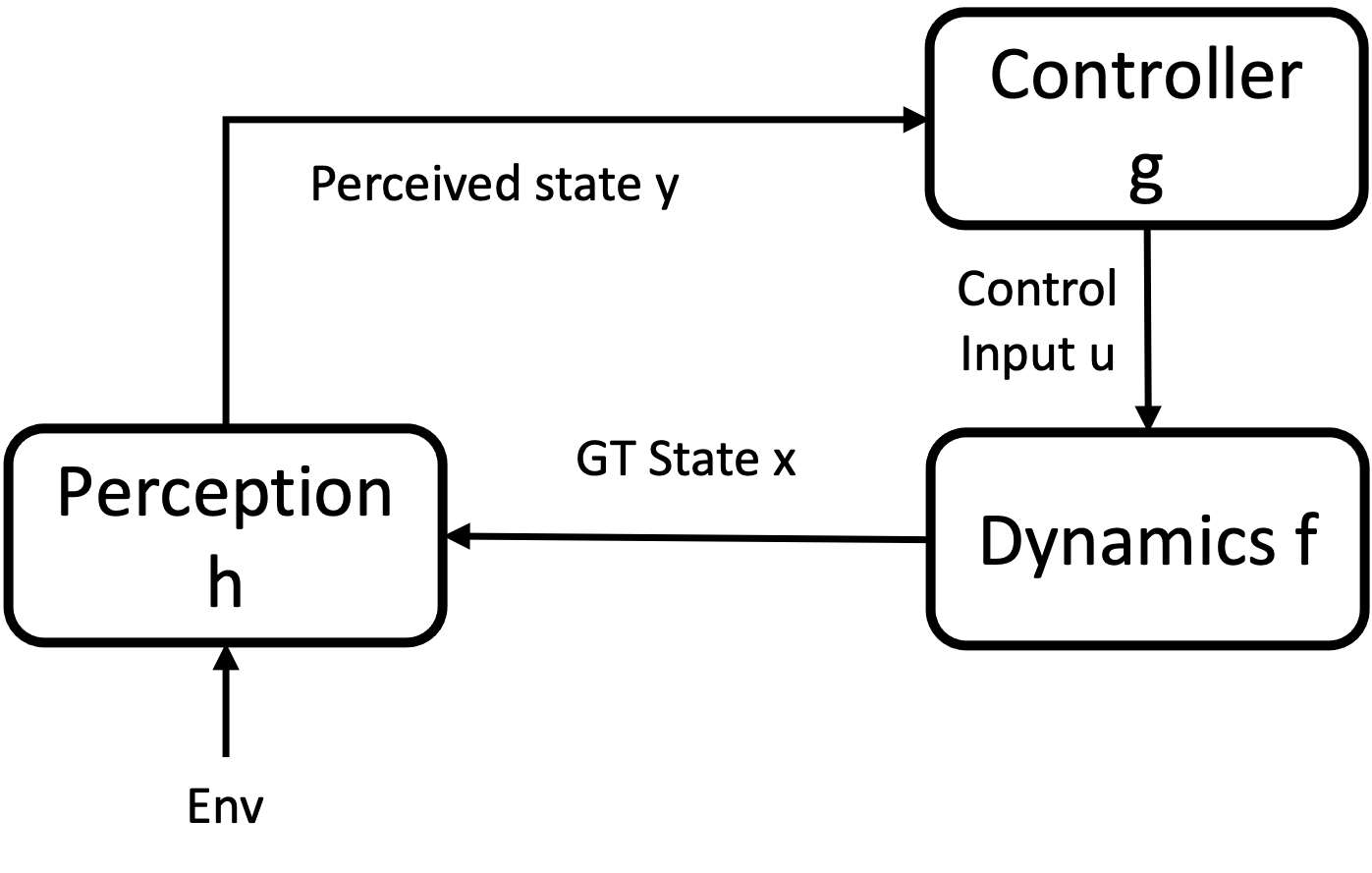}
    \caption{This is the block diagram of the closed loop system $S$, which is mathematically described in Equation~\ref{eq:sys}}
    \label{fig:original-system}
\end{figure}

\begin{figure}
    \centering
    \includegraphics[width=\linewidth]{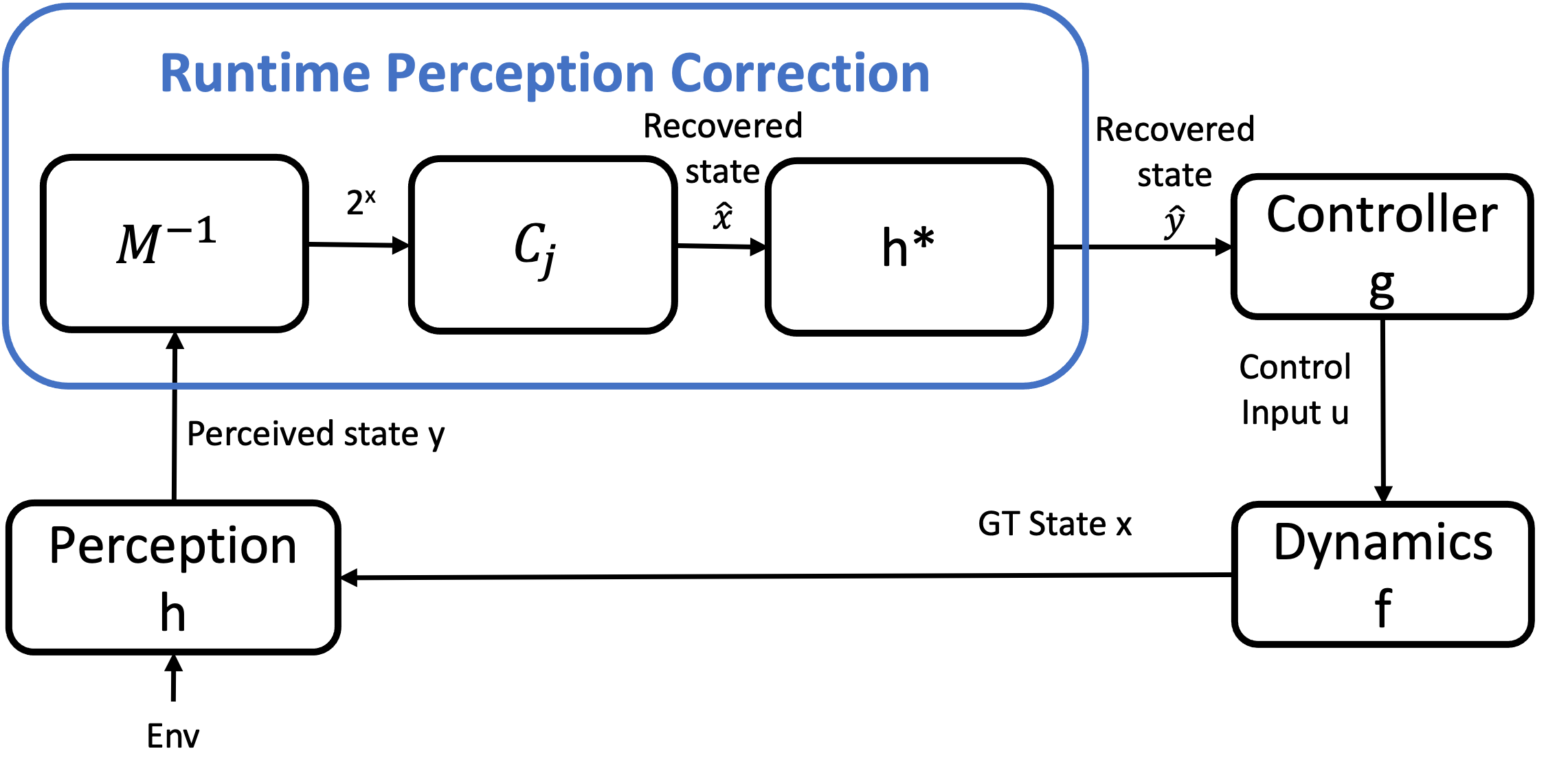}
    \caption{This is the block diagram of the new closed-loop system $S_{MJ}$ with our runtime perception correction module from $y$ to $\hat{y}$, which is mathematically described at Equation~\ref{sys:correct_sys}
    }

    \label{fig:perception_correct}
\end{figure}

\subsection{Constructing PPC from data}
\label{sec:bnn}
Ideally, PPC should achieve perfect conformance, i.e., Proposition~\ref{prop:ppc-ph} should always hold. However, this perfect conformance requirement of PPC might be too strong for real autonomous system involving perception. In this paper, instead of directly learning $M^{-1}$ from $M$, we propose a method to empirically learn a PPC from data using Bayesian Neural Network(BNN). We choose BNN as our model architecture due to its ability to model uncertainty and its ability to generalize from limited training data.

Let dataset $\D:=\{(y_i, (x_i, e_i))\}$, where $y_i$ is the input data, $(x_i, e_i)$ is the output label. Traditional neural networks output a point estimate for a given input, that is, they produce one deterministic value (or vector of values) given an input vector, i.e. $(x, e) = f_{\theta^*}(y)$, where f is the neural network parameterized by a deterministic optimal $\theta^*$, and $\theta^*$ in practise is obtained by conducting gradient descend to minimize the loss function.
\begin{equation}
    \theta^* = \argmin_{\theta} \sum_{(x_i, e_i, y_i) \in \D} \mathcal{L}(f_\theta(y_i), (x_i, e_i))) 
\end{equation}

A Bayesian Neural Network(BNN), in contrast, produces a distribution over possible outputs for an input vector, since the parameters of the BNN are also random variables, i.e. $\widetilde{x}, \widetilde{e} = f_{\widetilde{\theta}}(y)$ where $\widetilde{x}, \widetilde{e}, \widetilde{\theta}$ are both random variables. This probabilistic approach allows for the modeling of uncertainty, which can be essential in many applications, especially when decisions based on the output have significant implications, like in autonomous driving\cite{ijcai2017p661}. 

In this paper, we use a specific kind of BNN, where the prior distribution of the BNN weights follow a Gaussian distribution, i.e., $\theta \sim \N(\mu, \Sigma)$, the optimal weights are parameterized by $\mu^*, \Sigma^*$. Similarly to traditional neural networks, in practice, the optimal weight is found by doing gradient descend on the loss function.
\begin{equation}
\resizebox{0.48\textwidth}{!}{$
    \mu^*, \Sigma^* = \argmin_{\mu, \Sigma} \sum_{(x_i, e_i, y_i) \in \D} \mathcal{L}(f_\theta(x_i, e_i), y_i) - KL(p(\theta), p(\theta_0))
    $}
\end{equation}
where $\theta_0$ is a normal distribution, i.e., $\theta_0 \sim \N(0, 1)$. The KL divergence is added to to the loss function to prevent the weight from drifting away from normal distribution. We use a Gaussian distribution as a prior for weight of the BNN, because the Gaussian prior acts as a natural form of regularization to avoid overfitting, and have been used as a common technique in training the BNN. \cite{NIPS2011_7eb3c8be, blundell2015weight, gal2016bayesian}.

During the inference time of the BNN, we can get a set of $\hat{x}, \hat{e}$ given y, simply by sampling the querying the trained BNN N times. We include both $x$ and $e$ during training process to increase the training accuracy. But for the runtime perception correction problem, we only need to infer for $x$, so we will drop the $e$ dimension when making up $M^{-1}$ as follows:

\begin{equation}
    M^{-1}(y) = \{x_i | (x_i, e_i) \sim f_{\theta^*}(y)\}_{i=1}^N
\end{equation}

\subsection{Risk Heuristic}
\label{sec:riskh}

We introduce a notion of risk which defines a monotonic function over states with respect to the safety property. A similar notion ``Monotonic Safety'' was used in~\cite{9797566} to remove the uncertainty of non-deterministic models for Statistic Model Checking. 

\begin{definition}
    For a given control invariant set $R \subseteq \X$, a {\em risk function\/} $J: \X \mapsto \R_{\geq 0}$ assigns a real value to each state such that for any two safe states $x_1, x_2 \in \Safe$, if $J(x_1) \geq J(x_2)$ and there exists an $R$-preserving  control $\bar{u} \in \U$ for $x_1$, then $\bar{u}$ also preserves $R$ for $x_2$, that is, $f(x_2, \bar{u})) \in R.$ 

    \label{def:risk-function}
\end{definition}
Informally, if a more risky state preserves the invariant $R$, then a less risky state also preserves the same invariant by taking the same control action.  In Examaple~\ref{example:autonomous-vehicle}, consider the ego vehicle following lanes on curvy highway and two states $x_1$=(${p_E}_1, v_E = 5, \theta=0, d_L =\frac{L}{2}$), $x_2$=(${p_E}_2, v_E = 5, \theta=0, d_L =0$), and a risk function $J(x) = |d_L|$, which measures the risk as distance to lane center. It's obvious that $J(x_1) \geq J(x_2)$ since $x_1$ is on the lane boundary while $x_2$ is in the lane center, and if the next state of $x_1$ with full throttle $u = (t=1, b=0, s=0)$ under dynamics is in $R$, then $x_2$ with the same full throttle control will also stay in $R$.

Using the risk heuristic $J$, we can define a corresponding function $C_J$ that chooses the riskiest  state from the intersection of the PPC $M^{-1}$ and the invariant set $R$. Here we take the intersection because, according to definition~\ref{def:control-invariant-set}, states outside $R$ cannot guarantee the existence of control values that will preserve the invariant set $R$. 

   \begin{equation}
        C_J(M^{-1}(y)) = \argmax_{x \in  M^{-1}(y) \cap R} J(x).    
        \label{eq:cj}
    \end{equation}

Finally, we  apply the perfect observer $h^*$ and subsequently the control function $g$ to the state returned by $C_j(M^{-1}(y))$ to compute the control input to the plant. The evolution of the resulting corrected closed-loop system $\corrSys$ is given by (shown in Figure~\ref{fig:perception_correct}):
\begin{equation}
    x_{t+1} = f(x_t, g(h^*(C_J(M^{-1}(h(x_t, e)))))).
    \label{sys:correct_sys}
\end{equation} 

We claim that $\corrSys$ indeed preserved the safety invariant $R$ provides $h$ conforms to $M$, 
and therefore, solves the runtime perception correction problem. 

\begin{theorem}
    Given a preimage of a perception contract $M^{-1}$ for the actual perception function $h$ such that $h$ conforms to $M$ and a risk heuristic  $J$ for $R$, the corrected system $\corrSys$ described by Equation~(\ref{sys:correct_sys}) preserves the  invariant $R$ for $\E' \in \E$.
    \label{thm:safety}
\end{theorem}

\begin{proof}
    Consider any $x_t \in R$ to be the input of $h$. From Proposition~\ref{prop:ppc-ph}, we know that for any $x_t \in \X, e \in \E'$ with $h(x_t,e) = y$, if $h$ conforms to the perception contract $M$, we have
    \begin{equation*}
        x_t \in M^{-1}(h(x_t,e)) 
    \end{equation*}
    Since $x_t \in R$ then $M^{-1}(h(x_t, e)) \cap R \neq \emptyset$. From Equation~\ref{eq:cj}, we have the output of the heuristic function, calling it $\hat{x}$, satisfying
    \begin{equation*}
        \hat{x} = C_J(M^{-1}(h(x_t,e))) \in (M^{-1}(h(x_t, e)) \cap R) \subseteq R. 
    \end{equation*}
     From Assumption~\ref{assump:g_h*}, since $\hat{x} \in R$, we know that there exists $g$ such that 
    \begin{equation*}
        f(\hat{x}, g(h^*(\hat{x})) \in R.
    \end{equation*}
    From Equation~\ref{eq:cj},  since $\hat{x}$ is chosen to maximize the Risk Function $J$ within $(M^{-1}(y) \cap R)$, then  $J(\hat{x}) \geq J(x_t)$. From Definition~\ref{def:risk-function}, we hence have that 
    \begin{equation*}
        x_{t+1} = f(x_t, g(h^*(\hat{x})) \in R \subseteq Safe
    \end{equation*}
    thus concluding the proof. 
\end{proof}

\vspace{1em}

%% file: section/6.experiments.tex
\section{Autonomous Cruising Control}
\label{sec:case-study}
In this section, we introduce the details of an autonomous cruise control system (ACC) which we will analyze in Section~\ref{sec:exp}.

\subsection{Autonomous Cruise Control Problem}

The driver assistance feature we study is a combination of lane keeping and adaptive cruise control. Similar autonomy features commercially go by other names such as AutoPilot, Travel Assist, AutoCruise, etc. In typical operation, the ego vehicle moves at a set speed  behind a lead vehicle. Both vehicles follow (possibly curving) lanes. If the lead vehicle slows down then the ego vehicle has to maintain safe separation. The ego vehicle's autonomy pipeline uses vision-based perception. We explore three lane configurations: Track 1, a lane with left curve; Track 2, a lane with right curve; Track 3, a straight lane.

The state of the whole system includes ($p_E, p_F, v_E$, $v_F, \theta, d_L, d_F$), where $p_E$, $p_F$ are the pose (position and heading) of the ego vehicle and the leading vehicle respectively, $v_E, v_F$ are the velocity of the ego vehicle and the leading vehicle respectively, $\theta$ is the angle between ego's heading and the lane's heading, $d_L$ is ego vehicle's distance to lane center(also known as cross track error), $d_F$ is ego's distance to the leading vehicle. The unsafe set is defined as states where the ego vehicle is out of the lane boundaries or there is a collision, i.e. $\Unsafe = \{(p_E, p_F, v_E, v_F, \theta, d_L, d_F) | d_L \geq \frac{L}{2} \lor d_F \leq W\}$, where $L$ is the lane width and $W$ is the vehicle length. 
The observations given by sensors and ML-based perception will consist of $(\theta, d_L, d_F)$, which have the same meaning as the state variables with the same names above.

We use CARLA~\cite{DBLP:journals/corr/abs-1711-03938} to create and run the scenarios. CARLA is an open-source platform designed specifically to support the development and validation of autonomous driving systems. For the ego vehicle and the leading vehicle, we use the Carla built-in Tesla Model 3 vehicle's dynamics.
We use a reachability tool~\cite{DBLP:journals/corr/FanQM017} to approximate the invariant $R$.

A key advantage of CARLA lies in its ability to create realistic and diverse real-world scenarios, as well as different weather conditions (ranging from clear skies and rain to fog and snow) that could affect the sensor readings. During runtime testing, we consider six  weather conditions ranging from demanding environments such as late night, heavy fog, and rain to clear skies (optimal driving conditions): Weather 1 to 5 represent decreasing levels of fog and rain, with Weather 6 characterized by clear skies. We intentionally chose 5 extreme weather conditions to assess their impact on perception $h$.

\paragraph*{Perception $h$}
We use two learning-based perception module to detect the front vehicle and the lane: YOLO v8n~\cite{reis2023realtime} and LaneNet~\cite{10.1109/IVS.2018.8500547}. YOLO v8n, an evolution of the 'You Only Look Once' series, is utilized for its swift and accurate object detection capabilities, particularly for identifying front leading vehicles. It's known for offering high detection accuracy while ensuring minimal latency. On the other hand, LaneNet is employed specifically for its prowess in lane detection. This architecture combines semantic segmentation with instance segmentation to precisely distinguish between individual lane lines, even in challenging conditions. Together, YOLO v8n and LaneNet form a routine perception backbone, ensuring our autonomous vehicle is consistently aware of its surroundings, which is also a commonly adopted by the research community for autonomous vehicle systems~\cite{10.1109/ICRA48506.2021.9561747}. The two modules will output location of leading vehicle and lanes in the camera frame. Together with a depth camera and the known intrinsic and extrinsic matrix of the camera sensor, we can get the observation $y := (\theta, d_L, d_F)$. We fine-tuned the two models on a customized Carla image dataset.

\paragraph*{Controller $g$}
We experiment with 3 lateral  and 1 longitudinal controller. The controllers are implementations of racing controllers submitted by leading participants of the GRAIC competition~\cite{osti_10296575}. All controllers  take  the same observations, and produce the  control values, namely throttle, brake and steering.

The lateral controllers we have from GRAIC are a modified verison of Pure Pursuit\cite{DBLP:journals/corr/abs-2111-08873}, a modified version of Stanley\cite{stanley} and a simplified kinematic steering control. The Pure Pursuit algorithm computes the required steering angle based on the vehicle's lookahead distance; the Stanley controller computes the steering based on the vehicle's orientation relative to the path and cross-track-error; the simplified kinematic steering control uses simple geometric solution to the path-following problem.

For longitudinal control, we utilize a PID controller, designed to maintain a constant vehicle speed with high precision. In scenarios demanding instant decisions, such as potential collisions, we incorporate the Responsibility Sensitive Safety (RSS) formula. This ensures the vehicle brakes promptly and safely, taking into account both the vehicle's dynamics and the surrounding environment.

The vehicle controller $g$ is a combination of both lateral control and longitudinal control. We name the combination of Pure Pursuit and PID Controller $C_1$, the combination of Stanley and PID $C_2$, the combination of kinematic control and PID $C_3$.

The combination of three road geometries, six set of weather, and three controllers define 54 different ACC {\em scenarios}. Each scenario can be modeled as a closed loop system $S$ satisfying Equation~(\ref{eq:sys}) if we are using perception $h$. Moreover, each scenario can also be modeled as a new closed loop system $S_{MJ}$ satisfying Equation~(\ref{sys:correct_sys}) if we use both perception $h$ and our runtime perception correction ($M^{-1}, C_J, h^*$).

\section{Experimental Evaluation}
\label{sec:exp}
We apply runtime perception correction to the ACC scenarios described above and have following observations.

\subsection{Construction of Preimage Perception Contract $M^{-1}$}

To obtain training data for generating preimage perception contract, we run a safe controller across 20 sets of different weather conditions and 3 different lane configurations to collect pairs of ground truth states $x$ and observation $y=h(x, e)$. We run the controller for 3 hours and collect 5k set of such pairs to make up the dataset $\D = \{(y_i, (x_i, e_i))\}$, where $y_i$ is the input feature, and $(x_i, e_i)$ is the output label.

We divide the dataset into 80\% training dataset, and the rest 20\% for validation, as this is a common practice in neural network training. We trained the data for 3000 epoches with batch size 32 on a BNN, discussed at Section~\ref{sec:bnn}. The BNN has 3 hidden layers with 32, 128, 16 neurons respectively. We assume the prior distribution of each layer's weights follow a Gaussian distribution $\N(0, 0.3)$, we use learning rate 0.01. The layers, neurons and prior distributions are hyper-parameters subject to tuning. We find such hyper-parameters after empirical tryouts.

\subsection{Choice of Risk Function}

The high-level idea behind the specific risk function $ J: X \mapsto \mathbb{R}_{\geq 0}$ we chose in this case study is that the closer the vehicle is to unsafe state, the more risk the state will have. This function is monotonically increasing with respect to the inherent risk of the state.

We define the risk as the $L_\infty$ norm of the weighted inverse of the distance to the unsafe region, taken over each dimension of the state space.

\begin{equation}
    J(x) =  
    \begin{Vmatrix}
        \begin{bmatrix}
        \frac{w_1}{x[1] - x^u[1]} \\
        \frac{w_2}{x[2] - x^u[2]} \\
        \vdots \\
        \frac{w_n}{x[n] - x^u[n]} \\
        \end{bmatrix}
    \end{Vmatrix}_\infty 
    \label{eq:risk-function}
\end{equation}
where $x^u = (0, 0, 0, 0, \frac{\pi}{2}, \frac{L}{2}, W)$ is the unsafe boundary for the state variables $(p_E, p_F, v_E, v_F, \theta, d_L, d_F)$ and the corresponding weight $w = (0, 0, 0, 0, 0.2, 0.4, 0.4)$. $w$ is carefully chosen after empirically tryouts to satisfy Definition~\ref{def:risk-function}. The idea behind such choice of weight is because we want to assign high risks to vehicle states where it's too close to the leading vehicle (i.e. $d_F = W$ is where collision will happen) or cross track error is too high (i.e. $d_L = \frac{L}{2}$ is where out-of-lane will happen) or heading error is too large(i.e. $\theta = \frac{pi}{2}$ is where ego is almost perpendicular to the lanes).

\subsection{Runtime Perception Correction corrects 73\% unsafe scenarios} 
We first ran ego with perfect observer $h^*$ (i.e., ground truth observation provided by Carla simulator), and all 54 scenarios result in safe execution, which indicates that Assumption~\ref{assump:g_h*} holds for the three controllers we have. Next, we ran ego with perception module $h$ on the same 54 scenarios. We discovered that, ego safely finished 9 scenarios while the rest 45 scenarios resulted in an unsafe finish (e.g., either a collision with the front vehicle or steering out of lanes). Last, we ran the ego with $h$ complemented by our runtime perception correction module ($M^{-1}, C_J, h^*$), and we see that our method maintains safety in 9 scenarios where using $h$ alone is safe; furthermore, it can recover 33 out of the 45 (73 \% recover rate) unsafe scenarios in which $h$ failed. We looked into the log of 12 scenarios where our approach cannot recover, and we noticed that it's due to that $h^{-1}$ does not fully conform to the constructed PPC from data(i.e. $x = h^{-1}(y) \notin M^{-1}(h(x,e))$ for some $x \in \X, e \in \E'$ in these 12 scenarios) , which violated the conformance assumption of Theorem~\ref{thm:safety} and thus safety can not be guaranteed. Detailed results are provided in Table~\ref{table:correction}.

\begin{table}
    \centering
        \begin{tabular}{|c|c|c|c|c|c|c|c|}
\hline
& & \multicolumn{2}{c|}{Track 1} & \multicolumn{2}{c|}{Track 2} & \multicolumn{2}{c|}{Track 3} \\
\hline
Scenarios & & h & $M^{-1}$ & h & $M^{-1}$ & h & $M^{-1}$ \\
\hline
\multirow{3}{*}{ Weather 1 } & C1 & X & X & X & X & X & X \\
& C2 & X & X & X & \checkmark & X & \checkmark \\
& C3 & X & \checkmark & X & \checkmark & X & \checkmark \\
\hline
\multirow{3}{*}{ Weather 2 } & C1 & X & X & X & \checkmark & X & \checkmark \\
& C2 & X & \checkmark & X & \checkmark & X & \checkmark \\
& C3 & X & \checkmark & X & \checkmark & X & \checkmark \\
\hline
\multirow{3}{*}{ Weather 3 } & C1 & X & X & X & X & X & X \\
& C2 & X & X & X & \checkmark & X & \checkmark \\
& C3 & X & \checkmark & X & \checkmark & X & \checkmark \\
\hline
\multirow{3}{*}{ Weather 4 } & C1 & X & X & X & X & X & \checkmark \\
& C2 & X & X & X & \checkmark & X & \checkmark \\
& C3 & X & \checkmark & X & \checkmark & X & \checkmark \\
\hline
\multirow{3}{*}{ Weather 5 } & C1 & X & \checkmark & X & \checkmark & X & \checkmark \\
& C2 & X & \checkmark & X & \checkmark & X & \checkmark \\
& C3 & X & \checkmark & X & \checkmark & X & \checkmark \\
\hline
\multirow{3}{*}{ Weather 6 } & C1 & \checkmark & \checkmark & \checkmark & \checkmark & \checkmark & \checkmark \\
& C2 & \checkmark & \checkmark & \checkmark & \checkmark & \checkmark & \checkmark \\
& C3 & \checkmark & \checkmark & \checkmark & \checkmark & \checkmark & \checkmark \\
\hline
\end{tabular}

    \caption{
    Safety evaluation under a variety of road conditions, weather and controllers. 
     "X" indicates unsafe scenarios while "\checkmark" signifies ego can finish the scenario safely}
    \label{table:correction}
\end{table}

\begin{figure*}[tbp]
    \centering

    \begin{subfigure}[b]{0.45\textwidth}
        \includegraphics[width=\textwidth]{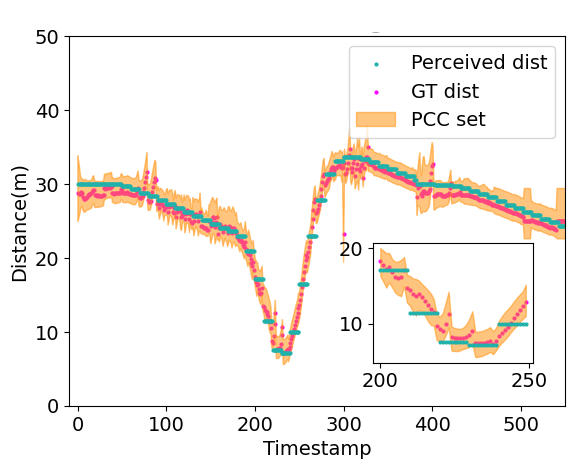}
        \caption{}
        \label{fig:ppc-testing-result-1}
    \end{subfigure}
    \hfill
    \begin{subfigure}[b]{0.45\textwidth}
        \includegraphics[width=\textwidth]{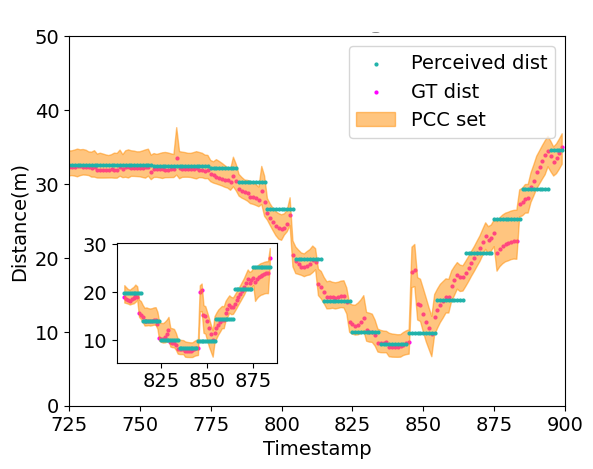}
        \caption{}
        \label{fig:ppc-testing-result-2}
    \end{subfigure}
    \label{fig:ppc-testing-result}

    \caption{The two graphs show PPC's prediction on one of the state dimension(distance to front vehicle $d_F$). Red dot indicates ground truth distance given by $h^*$, while green dots indicate observed distance given by $h$. Shaded area represented the constructed PPC $M^{-1}$}.
    
\end{figure*}

\subsection{Correction Intervention are minimally evasive}

\begin{figure}
    \centering
    \includegraphics[width=\linewidth]{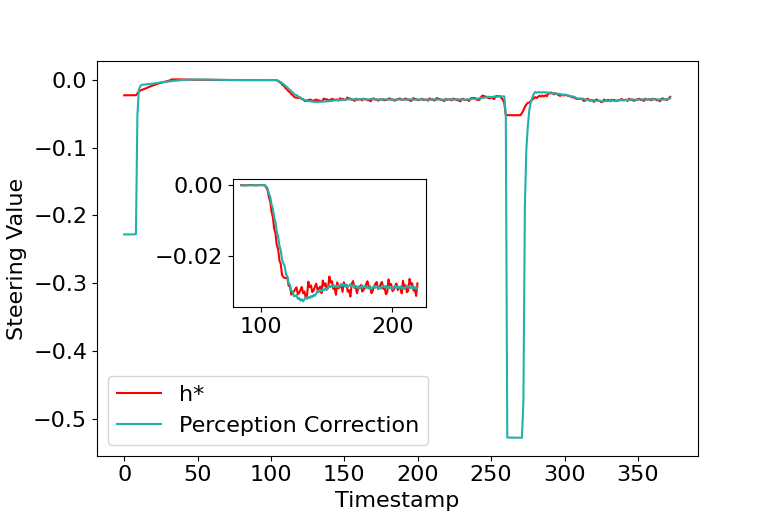}
    \caption{Vehicle's steering profile on one particular scenario.}
    \label{fig:correction-comparison}
\end{figure}

While the runtime perception correction module helps to maintain safety, it can also be integrated without inducing overly unnecessary behaviors, such as frequent stops or excessive adjustments. We look at the vehicle’s steering profile of one particular example(Track 1, Weather 6, C3) in Figure \ref{fig:correction-comparison}. We noticed that most of the time, our runtime perception correction results in a similar steering value as ones produced by the perfect observer $h^*$. Most of the time, runtime perception correction module does not result in unnecessary movement except for timestamp 270 to 280, where the curvature of the map slightly increase and our PPC captures a large uncertainty in perception errors, thus results in a large control value. 

Given the large number of experiments, analyzing each scenario by directly comparing control values can be both time-consuming and potentially misleading due to temporal shifts in these values. As such, we employ a more informative quantitative metric -- time to finish the scenario, which helps in gauging whether the vehicle engages in unnecessary behaviors. Due to the fact that each track has different length and potentially different time to finish, we compared the time to finish using perception $h$ and our runtime perception correction module ($M^{-1}, C_J, h*$) with time to finish with perfect observer $h^*$. Then in Table \ref{tab:time-to-finish}, we show the percentage of increase in time to finish using our runtime perception correction module. As presented, the percentage increases in completion time are, for the most part, negligible. The average percentage increase in time to finish the 33 scenarios is 2.8\%. This shows that our approach is minimally invasive to the original closed loop system.

\begin{table}[htbp]
    \centering

    \begin{tabular}{|c|c|c|c|c|}
    \hline
    & & Track 1 & Track 2 & Track 3 \\
    \hline
    \multirow{3}{*}{ Weather 1 } & C1 & X & X & X \\
    & C2 & X & +1.5\% & +4.1\% \\
    & C3 & +3.2\% & +2.6\% & +4.6\% \\
    \hline
    \multirow{3}{*}{ Weather 2 } & C1 & X & +1.0\% & +3.0\% \\
    & C2 & +1.3\% & +1.5\% & +3.6\% \\
    & C3 & +3.2\% & +2.6\% & +4.1\% \\
    \hline
    \multirow{3}{*}{ Weather 3 } & C1 & X & X & X \\
    & C2 & X & +1.5\% & +4.6\% \\
    & C3 & +3.2\% & +2.6\% & +5.2\% \\
    \hline
    \multirow{3}{*}{ Weather 4 } & C1 & X & X & +2.5\% \\
    & C2 & X & +1.5\% & +3.6\% \\
    & C3 & +3.2\% & +2.6\% & +4.1\% \\
    \hline
    \multirow{3}{*}{ Weather 5 } & C1 & +2.6\% & +1.0\% & +2.5\% \\
    & C2 & +1.3\% & +1.5\% & +3.6\% \\
    & C3 & +3.2\% & +2.6\% & +5.2\% \\
    \hline
    \multirow{3}{*}{ Weather 6 } & C1 & +2.6\% & +1.0\% & +2.5\% \\
    & C2 & +1.3\% & +1.5\% & +4.6\% \\
    & C3 & +3.2\% & +2.6\% & +4.1\% \\
    \hline
    \end{tabular}

    \caption{Percentage of increase in time to finish by comparing system use runtime perception correction ($M^{-1}, C_J, h*$) with system that uses perfect observer $h^*$. "X" indicates our runtime perception correction module cannot finish that scenario safely}
    
    \label{tab:time-to-finish}
\end{table}

\subsection{PPC can be empirically constructed by BNN}
We evaluated whether the constructed PCC is effective by applying the trained BNN model on scenarios not in its training data, and we see that for 91.2\% of the testing data, $h^{-1}$ conforms to the constructed $M^{-1}$ (i.e. $x = h^{-1}(y) \in M^{-1}(h(x, e))$ for $x, e$ in the testing dataset). We show an example of how PCC visualizes in Figure~\ref{fig:ppc-testing-result-1},~\ref{fig:ppc-testing-result-2}. We see that, for most of the time, the constructed PCC through BNN always include the ground truth, which satisfies the conformance property of $M^{-1}$. For example, in Figure~\ref{fig:ppc-testing-result-1}, between timestamp 200 and 250, although observed distance, generated by perception $h$, slightly differ from ground truth distance, the PCC we constructed(yellow shaded part) through BNN always contain the ground truth.

\vspace{1em}

%% file: section/7.discussions.tex
\section{Conclusions}
We presented our method, which leverages the preimage perception contract and a risk heuristic, to correct learning-based perception errors for safety during runtime, assuming perfect conformance. Empirically, our approach demonstrated the capability to rectify 73\% of unsafe Adaptive Cruise Control (ACC) scenarios stemming from perception errors, while minimizing unnecessary behavior.
\vspace{1em}

\subsection{Limitation}
Ideally, by theorem~\ref{thm:safety}, we should be able to recover any  unsafe scenarios caused by noisy perception module. However, the empirical result only showed 73\% success. This is due to the fact that constructing the Preimage Perception Contract $M^{-1}$ through BNN does not guarantee 100\% conformance. In other word, Proposition \ref{prop:ppc-ph} might not always hold; however, in the same time, constructing $M^{-1}$ using any data-driven method(E.g. quantile regression) will face the same issue. 

\subsection{Future Directions}

In this paper, we extensively test multiple scenarios (different curvature, weather, and road conditions) to find particular five adversarial scenarios where perception module fails to ensure safety. However, manually searching through the whole state space to find an adversarial example is time-consuming and inefficient. Moreover, the BNN-based training method does not always guarantee conformance; therefore, alternate architecture or model to construct teh PPC could be explored. Our next step is to use the preimage perception contract to infer potential environment variables that can might lead to unsafe scenarios. This could help indicate the weakness of the controller.